\documentclass[letterpaper, 10 pt, conference]{ieeeconf}  

\IEEEoverridecommandlockouts                              
\let\labelindent\relax

\usepackage{enumitem}
\usepackage{comment}
\usepackage{cite}
\usepackage{amsthm}
\usepackage{amsmath,amssymb,amsfonts}

\usepackage{algorithm}
\usepackage{svg}
\usepackage{algpseudocode}
\usepackage{graphicx}
\usepackage{textcomp}
\usepackage{xcolor}
\def\BibTeX{{\rm B\kern-.05em{\sc i\kern-.025em b}\kern-.08em
    T\kern-.1667em\lower.7ex\hbox{E}\kern-.125emX}}
    \renewcommand{\baselinestretch}{0.96}
\usepackage{quantikz}
\usepackage{tikz}
\usepackage{caption}
\usepackage{hyperref}
\usepackage{multicol}
\usepackage{subfig} 
\def \xbf {\mathbf{x}}

\def \Zbf {\mathbf{Z}}
\def \Ebb {\mathbb{E}}
\def \Rbb {\mathbb{R}}
\def \Hbb {\mathbf{H}}
\def \Wbf {\mathbf{W}}
\def \thetabf{\boldsymbol{\theta}}
\def \vec {\text{vec}}
\def \obf {\boldsymbol{0}}

\def \Ical{\mathcal{I}}
\newcommand\scalemath[2]{\scalebox{#1}{\mbox{\ensuremath{\displaystyle #2}}}}
\newtheorem{theorem}{Theorem}
\newtheorem{assumption}{Assumption}
\newtheorem{definition}{Definition}
\newtheorem{lemma}{Lemma}
\newtheorem{remark}{Remark}
\begin{document}

\title{\LARGE \bf Neural Contextual Bandits Under Delayed Feedback Constraints}
\author{Mohammadali Moghimi, Sharu Theresa Jose, and Shana Moothedath
\thanks{M. Moghimi and S.T. Jose are with the School of Computer Science, University of Birmingham, Birmingahm B15 2TT, UK
        {\tt\small mmoghimi486@gmail.com, s.t.jose@bham.ac.uk}}%
\thanks{S. Moothedath is with the Department of Electrical Engineering, Iowa State University, Ames, IA 50011, USA
        {\tt\small mshana@iastate.edu}}%
}

\maketitle

\begin{abstract}
This paper presents a new  algorithm for neural contextual bandits (CBs) that addresses the challenge of delayed reward feedback, where the reward for a chosen action is revealed after a random, unknown delay.
This scenario is common in applications such as online recommendation systems and clinical trials, where reward feedback is delayed because the outcomes or results of a user’s actions (such as recommendations or treatment responses) take time to manifest and be measured. The proposed algorithm, called \textit{Delayed NeuralUCB}, uses upper confidence bound (UCB)-based exploration strategy. Under the assumption of independent and identically distributed sub-exponential reward delays, we derive an upper bound on the cumulative regret over $T$-length horizon, that scales as $O(\tilde{d} \sqrt{T \log T}+  \tilde{d}^{3/2} D_+\log(T)^{3/2})$ where $\tilde{d}$ denotes the effective dimension of the neural tangent kernel matrix, and $D_+$ depends on the expected delay $\Ebb[\tau]$. We further consider a variant of the algorithm, called Delayed NeuralTS, that uses Thompson Sampling based exploration.
Numerical experiments on real-world datasets, such as MNIST and Mushroom, along with comparisons to benchmark approaches, demonstrate that the proposed algorithms effectively manage varying delays and are well-suited for complex real-world scenarios.
\end{abstract}


\section{Introduction}

The stochastic contextual bandit (CB) problem has gained immense interest in recent years due to its application in various domains, including healthcare, finance, and recommender systems \cite{durand2018contextual, shen2015portfolio,mcinerney2018explore,lin2024distributed, lin2022stochastic}. The CB is a sequential decision-making problem where, in each round, the agent (or decision-maker) is presented with $K$ actions and associated contextual information. Based on the observed contexts, the agent chooses an action and receives a stochastic reward from the environment. The reward generating mechanism is unknown to the agent, and the agent only observes the received reward. The agent's objective is to maximize the expected cumulative rewards over a specified horizon of length $T$, or equivalently to minimize the cumulative regret over the horizon.

Many provably efficient algorithms have been developed for stochastic CBs assuming linear reward models -- linear upper confidence bound (LinUCB \cite{chu2011contextual}) and linear Thompson Sampling (LinTS \cite{agrawal2013thompson}) -- as well as generalized linear models \cite{filippi2010parametric}, and neural reward models (NeuralUCB \cite{zhou2020neural} and NeuralTS \cite{zhang2020neural}). A common feature of all these approaches is that they assume \textit{immediate} reward feedback from the environment. However, this is not true in many practical settings. For instance, in an online recommendation system (e.g., Netflix), thousands of movie recommendations are suggested to the user in milliseconds. However, it may take hours or even days for the user to respond to them. The same holds true in clinical trials, where the response time to administered treatments can vary significantly. 

Recent research \cite{joulani2013online,pike2018bandits,vernade2017stochastic,mandel2015queue} aims to model such scenarios by studying bandits with \textit{delayed reward feedback}: At each round $t $, the environment generates a random delay, unknown to the agent, and schedules an observation time for the reward $r_{t,a_t}$ corresponding to the action $a_t$ chosen at $t^{\rm th}$ round. Theoretically, the presence of delayed feedback introduces uncertainty regarding the information available to the agent at any given time due to missing rewards, which complicates the standard regret guarantees. Consequently, recent works have focused on developing algorithms tailored to the delayed-reward feedback setting. 

In this context, relevant recent works include \cite{zhou2019learning} and \cite{blanchet2024delay}, which develop novel algorithms for \textit{generalized linear contextual bandits} with delayed rewards. These algorithms boost the exploration bonus when the number of missing rewards is large, resulting in regret guarantees that scale as $\tilde{O}(d\sqrt{T}+\sqrt{dT(\Ebb[\tau]+M))}$, where $d$ is the dimension of the context vector, $\Ebb[\tau]$ is the mean of the unknown reward, and $M$ is a delay-dependent parameter. However, the delay terms in their regret scale with $\sqrt{T}$. Subsequent work \cite{howson2023delayed} proposed a simpler algorithm for generalized linear bandits, eliminating the scaling of delay with $\sqrt{T}$ and achieving a regret bound of $\tilde{O}(d\sqrt{T}+d^{3/2}\Ebb[\tau])$. This, however, is achieved at the cost of increased dependence of the regret on the dimension $d$. Here $\tilde{O}$ hides the logarithmic terms.

Despite these advances, most research on delayed reward feedback has focused on generalized {\em linear} CBs, which are quite restrictive to apply to real-world scenarios with complex reward functions. Neural CBs that leverage the strong representation power of deep neural networks are better suited for modeling complex reward functions \cite{zhou2020neural,zhang2020neural,mandel2015queue,kassraie2022neural,xu2020neural,zhu2023scalable,salgia2023provably}. 
In this work, inspired by \cite{blanchet2024delay} and \cite{howson2023delayed},  we introduce  new neural CB algorithms that can efficiently handle delayed reward feedback.

The following are our main contributions:
\begin{enumerate}
    \item We propose a \textit{Delayed NeuralUCB} algorithm and its variant \textit{Delayed NeuralTS} algorithm for neural CBs with delayed reward observation.
    \item  For the Delayed NeuralUCB algorithm, we derive a novel upper bound on the cumulative regret over $T$-length horizon, that scales as $O(\tilde{d} \sqrt{T \log T}+  \tilde{d}^{3/2} D_+\log(T)^{3/2})$, under the assumption of independent and identically distributed sub-exponential reward delays. Here, $\tilde{d}$ denotes the effective dimension of the neural tangent kernel gram matrix, and $D_+$ depends on the expected delay $\Ebb[\tau]$. Importantly, the impact of delay does not scale linearly with $T$.
    \item We conduct extensive experiments and compare our approach against multiple benchmarks on real-world MNIST and Mushroom datasets with various delay distributions. The results demonstrate that our proposed algorithms effectively handle different delay scenarios.
\end{enumerate} 

{\bf Organization:} We introduce the delayed reward feedback model for neural CBs in Section~\ref{sec:problemsetting}. Section~\ref{sec:algorithm} introduces the proposed Delayed NeuralUCB algorithm and derives an upper bound on the cumulative regret. Section~\ref{sec:proof} discusses the proof of the regret bound. Section~\ref{sec:experiments} introduces Delayed NeuralTS algorithm and provides numerical experiments evaluating the performance of our proposed algorithms. Finally, we conclude in Section~\ref{sec:conclusion}.
  
\section{Problem Setting}\label{sec:problemsetting}
In this section, we first introduce the neural CB setting and then detail the delayed feedback environment under study.
\subsection{Neural Contextual Bandits}\label{sec:neuralcontextualbandits}
Consider the stochastic $K$-armed CB problem. At each round $t \in [T]$, the agent observes the \textit{context} consisting of $K$ feature vectors $\{\xbf_{t,a} \in \Rbb^d: a=1,\hdots,K\}$ corresponding to each of the $K$ arms. Based on the observed context, the agent selects an action $a_t$ and receives a reward $r_{t,a_t}$.  We assume that the reward is generated as,
\begin{align}
    r_{t,a_t}= h(\xbf_{t,a_t}) + \xi_t, \label{eq:true_reward_model}
\end{align}where $h(\cdot)$ is an unknown mean-reward function satisfying $0 \leq h(\xbf) \leq 1$, and $\xi_t$ is a zero-mean noise.
The goal of the agent is to design an action-selection policy that minimizes the cumulative regret over $T$ rounds,
\begin{align}
    R_T=\Ebb\Bigl[\sum_{t=1}^Tr_{t,a_t}- r_{t,a_t^*} \Bigr],
\end{align} defined with respect to the optimal action $a^*_t = \arg \max_{a \in [1,\hdots,K]}h(\xbf_{t,a})$ that maximizes the mean-reward.

Neural contextual bandits \cite{zhang2020neural, zhou2020neural} estimate the unknown mean reward function $h(\cdot)$ in Eq.~\eqref{eq:true_reward_model} using a fully connected neural network with depth $L\geq 2$ as
\begin{align}
    f(\xbf;\thetabf)=\sqrt{m} \Wbf_L \sigma\Bigl(\Wbf_{L-1}\sigma \Bigl(\hdots \sigma(\Wbf_1(\xbf)) \Bigr) \Bigr),
\end{align}where $\sigma(x)=\max\{x,0\}$ is the rectified linear unit (ReLU) activation function, $\Wbf_1 \in \Rbb^{m \times d}$, $\Wbf_i \in \Rbb^{m \times m}$ for $2\leq i\leq L-1$ and $\Wbf_L \in \Rbb^{m \times 1}$, with $\thetabf=[\vec(\Wbf_1)^{\top},\hdots, \vec(\Wbf_L)^{\top}]^{\top} \in \Rbb^{p} $ with $p=m+md+m^2(L-1)$. Furthermore, we denote the gradient of the neural network function by $g(\xbf;\thetabf)=\nabla_{\thetabf}f(\xbf;\thetabf) \in \Rbb^p.$
\subsection{Contextual Bandits with Delayed Feedback}
We now consider the setting where in each round $t$, the agent does not immediately observe the reward corresponding to the chosen action $a_t$. Instead,
in each $t$-th round,
\begin{itemize}
    \item The agent receives the context and selects action $a_t$.
    \item Unbeknownst to the agent, the environment generates a random delay $\tau_t \sim p_{\tau}$, with $\tau_t \in [0,\infty)$, a random reward $r_{t,a_t}$ as in Eq.~\eqref{eq:true_reward_model}, and then schedules an observation time of the reward as $\lceil t+\tau_t \rceil$.
    \item The agent receives delayed rewards from its previous actions: $\{(s,r_{s,a_s}): t-1 <s+\tau_s \leq t\}$. Note here that when the agent receives the delayed reward, it knows the time index when the corresponding action was taken.
\end{itemize}
Thus, at the time of making decisions at the $t^{\rm th}$ round, the learner has access to the rewards revealed until time $t-1$, i.e., the set of past rewards $\{r_{s,a_s}: s+ \tau_s \leq t-1\}$. 
In standard bandit learning, the reward parameters are estimated iteratively using the information from past actions and rewards. However, in the delayed setting, since the agent has access to only a subset of the data and the regret at $t$ depends on all past actions, it creates a limited information scenario where the agent must make decisions based on partial feedback. 
Based on this information structure, we define the $\sigma$-algebra generated by the set of observed information at the beginning of round $t$ as
\begin{align*} \mathcal{F}_{t-1} &=\sigma\Bigl(\{(\{\xbf_{s,a}\}_{a=1}^K,a_s,r_{s,a_s}): s+\tau_s \leq t-1\} \nonumber \\& \cup \{(\{\xbf_{s,a}\}_{a=1}^K,a_s): 1 \leq s \leq t-1, s+\tau_s >t-1\} \Bigr). \end{align*}
\subsection{Assumptions}
We make the following key assumptions on the reward noise variable $\xi_t$ in Eq.~\eqref{eq:true_reward_model} and the reward delay $\tau_t$.
\begin{assumption} \label{assum:1}
    Let $R \geq 0$. Then, the moment generating function of the reward noise distribution conditioned on the observed information satisfies the following inequality:
    \begin{align}
        \Ebb\Bigl[\exp (\gamma \xi_t) |\mathcal{F}_{t-1}\Bigr] \leq \exp(\gamma^2 R^2/2), \nonumber
    \end{align} for all $\gamma \in \mathbb{R}$.
\end{assumption}

In addition to the above assumption, we further assume that the delay random variables $\tau_t$ are $(\alpha,b)$-sub-exponential.
\begin{assumption}\label{assum:delays}
    The delays $\{\tau_t\}_{t=1}^{\infty}$ are non-negative, independent and identically distributed $(\alpha,b)$-sub-exponential random variables, i.e., their moment generating function satisfies the following inequality:
    \begin{align}
        \Ebb[ \exp(\gamma (\tau_t -\Ebb[\tau_t]))] \leq \exp(\alpha^2\gamma^2/2),
    \end{align} for some $\alpha,b\geq 0$, and all $|\gamma|\leq 1/b$.
\end{assumption}

The above assumption follows from \cite{howson2023delayed}  and includes the class of delay distributions with sub-exponential tails like $\chi^2$ and exponential distributions. 

\section{Delayed NeuralUCB Algorithm}\label{sec:algorithm}
\begin{algorithm}[t]
\caption{Delayed NeuralUCB Algorithm} \label{alg:neuralUCB}
\begin{algorithmic}[1]
\Require Number of rounds $T$, regularization parameter $\lambda$, exploration parameter $\nu$, confidence parameter $\delta$, norm parameter $S$, step size $\eta$, number of gradient descent steps $J$, network width $m$, network depth $L$
\State \textbf{Initialization:} Randomly initialize $\thetabf_0$ from Gaussian distribution
\State Initialize $\Zbf_0 = \lambda \mathbf{I}$, $\mathcal{I}_0=\emptyset$.
\For{$t = 1$ to $T$}
    \State Observe contexts $\{\xbf_{t,a}\}_{a=1}^K$
    \For{$a = 1$ to $K$}
        \State Compute $U_{t,a} = f(\xbf_{t,a}; \theta_{t-1}) + \gamma_{t-1} \sqrt{g(\xbf_{t,a}; \theta_{t-1})^\top \Zbf_{t-1}^{-1} g(\xbf_{t,a}; \theta_{t-1}) / m}$
    \EndFor
    \State Select action $a_t = \arg \max_{a \in [K]} U_{t,a}$
    \State Play action $a_t$ and observe  rewards $\{r_{s,a_s}: s\leq t, t-1 < s+\tau_s \leq t\}$ revealed at round $t$.
    \State Update $\scalemath{0.9}{\mathcal{T}_t = \{s: s\leq t, t-1<s+\tau_s \leq t\},$  $\Ical_t = \Ical_{t-1} \cup \mathcal{T}_t}$
    \State Update $\scalemath{0.9}{\Zbf_t = \Zbf_{t-1} + \sum_{s \in \mathcal{T}_t} g(\xbf_{s, a_s}; \thetabf_{t-1}) g(\xbf_{s, a_s}; \thetabf_{t-1})^\top / m}$
    \State Update $\scalemath{0.9}{\thetabf_t = \text{TrainNN}(\lambda, \eta, J, m, \{(\xbf_{s,a_s}, a_s, r_s)\}_{s \in \Ical_t}, \thetabf_0)}$
    \State Compute $\gamma_t$ as in Eq.~\eqref{eq:gamma_t}.
\EndFor
\end{algorithmic}
\end{algorithm}
\begin{algorithm}[t]
\caption{TrainNN$(\lambda,\eta,J,m,\{\xbf_{s,a_s}, a_s, r_s\}_{s \in \mathcal{I}_t},\thetabf_0)$ }\label{alg:TrainNN}
\begin{algorithmic}[1]
\Require Regularization parameter $\lambda$, step size $\eta$, number of gradient descent steps $J$, dataset $\{\xbf_{s,a_s}, a_s, r_s\}_{s \in \mathcal{I}_t}$, previous network parameters $\thetabf_{t-1}$
\State \textbf{Initialization:} Initialize network parameters $\thetabf_0$
\State Define $\mathcal{L}(\thetabf)= \sum_{s \in \mathcal{I}_t} (f(\xbf_{s,a_s},a_s,\thetabf)-r_s)^2/2+m\lambda \Vert \thetabf-\thetabf_0\Vert_2^2/2$
\For{$j = 1$ to $J$}
    \State Update the parameters: 
    $
    \thetabf_j \leftarrow \thetabf_{j-1} - \eta \nabla \mathcal{L}(\thetabf_{j-1})
    $
\EndFor
\State Return $\thetabf_J$
\end{algorithmic}
\end{algorithm}

In this section, we present our Delayed NeuralUCB algorithm that accounts for random delays in reward observations. Described under Algorithm~\ref{alg:neuralUCB}, the proposed algorithm adapts the recently proposed NeuralUCB algorithm  \cite{zhou2020neural} to efficiently handle delayed reward feedback.    NeuralUCB algorithm leverages deep neural network to predict the mean rewards $f(\xbf_{t,a};\thetabf)$ for various actions based on the observed context-action-reward tuples, while leveraging upper confidence bound (UCB) to guide the exploration. 

Delayed NeuralUCB starts by initializing the neural network parameters $\thetabf_0$. This is done by randomly generating each entry of $\thetabf_0$ from an appropriate Gaussian distribution (please refer to \cite{zhou2020neural} for more details on exact initialization). At each $t^{\rm th}$ round, the algorithm observes the context vectors $\{\xbf_{t,a}\}_{a=1}^K$ corresponding to all the $K$ actions, and uses it  to compute an upper confidence bound $U_{t,a}$ (line 6). In the definition of $U_{t,a}$,  $g(\xbf_{t,a};\thetabf_{t-1})$ denotes the gradient function, while $\gamma_t$ is the UCB parameter defined as
\begin{align}
    \gamma_t &=\sqrt{1+C_1m^{-1/6}\sqrt{\log m}L^4|\Ical_t|^{7/6}\lambda^{-7/6)}} \Bigl( \nu \Bigl[\log  \frac{{\rm det}(\Zbf_t)}{{\rm det}(\lambda \mathbf{I})} \nonumber\\&+C_2 m^{-1/6}\sqrt{\log m}L^4|\Ical_t|^{5/3}\lambda^{-1/6}-2\log \delta\Bigr]^{1/2}  +\sqrt{\lambda S}\Bigr)\nonumber\\&+(\lambda+C_3|\Ical_t|L)\Bigl((1-\eta m\lambda)^{J/2}\sqrt{|\Ical_t|/\lambda} \nonumber\\&+m^{-1/6}\sqrt{\log m}L^{7/2}|\Ical_t|^{5/3}\lambda^{-5/3}(1+\sqrt{|\Ical_t|/\lambda}) \Bigr) \label{eq:gamma_t},
\end{align}where $\Ical_t =\{s: s+\tau_s \leq t\}$ is the set of time indices with rewards revealed until time $t$. $S$ is a norm parameter (made precise later) and $C_1,C_2,C_3>0$ are constants. Additionally, $m$ and $L$ respectively denote the  width and depth of the neural network, $\lambda$ denotes the regularization parameter in Algorithm~\ref{alg:TrainNN} while $\eta$ denotes the step size, and $\nu,\delta>0$ denote the exploration parameter and confidence parameter respectively. 

At round $t$, the action $a_t$ that maximizes the UCB is selected (line 8). 
Subsequently, the agent receives rewards $\{r_{s,a_s}: s\leq t, t-1 < s+\tau_s \leq t\}$ from past actions that are revealed at round $t$. The agent uses this to keep track of the set $\mathcal{T}_t= \{s:s\leq t, t-1<s+\tau_s \leq t\}$ that contains time stamps with reward information revealed at $t$ and updates $\Ical_t=\Ical_{t-1}\cup \mathcal{T}_t$. The agent updates the designer matrix $\Zbf_t$ using the context information available corresponding to the time stamps in $\mathcal{T}_t$. Finally, at the end of round $t$, the parameters $\thetabf_t$ of the neural network are updated using gradient descent (Algorithm~\ref{alg:TrainNN}) on the dataset $\{(\xbf_{s,a_s},a_s,r_s)\}_{s \in \Ical_t}$ with complete reward information. 

\begin{remark}
Delayed NeuralUCB differs from NeuralUCB in that the update of the designer matrix $\Zbf_t$ at $t^{\rm th}$ iteration (see line 11) happens only when at least one past reward is revealed during that iteration. 
 In our setting, unlike NeuralUCB, there can be rounds where no rewards are revealed, preventing the update of the design matrix.
 Additionally, the UCB parameter $\gamma_t$ in Eq.~\eqref{eq:gamma_t} depends on the total number $|\Ical_t|$ of past rewards revealed until time $t$, which is less than or equal to $t$.
\end{remark}

\subsection{Main Result: Regret Bound}
We now present our main result: an upper bound on the cumulative regret $R_T$ of the Delayed NeuralUCB algorithm. The regret analysis will be built on the theory of neural tangent kernel matrix following \cite{zhang2020neural}. To this end, we first define the neural tangent kernel matrix as below.

\begin{definition}[\cite{jacot2018neural}]
    Let $\{\xbf_i\}_{i=1}^{TK}$ denote the collection of all contexts $\{\xbf_{t,a}\}$. Define  for $i,j \in [TK]$ and $l \in [L]$,
    \begin{align*}
        \tilde{\Hbb}_{i,j}^{(1)}&=\Sigma_{i,j}^{(1)} = \langle \xbf_i, \xbf_j\rangle, \mathbf{A}_{i,j}^{(l)}=\begin{bmatrix}
            \Sigma_{i,i}^{(l)} & \Sigma_{i,j}^{(l)}\\
            \Sigma_{j,i}^{(l)} & \Sigma_{j,j}^{(l)}
        \end{bmatrix}     , \\
        \Sigma_{i,j}^{(l+1)}&=2 \Ebb_{u,v \sim \mathcal{N}(\mathbf{0}, \mathbf{A}_{i,j}^{(l)})}[\sigma(u)\sigma(v)],\\
        \tilde{\Hbb}_{i,j}^{(l+1)}&=2\tilde{\Hbb}_{i,j}^{(l)}\Ebb_{u,v \sim \mathcal{N}(\mathbf{0}, \mathbf{A}_{i,j}^{(l)})}[\sigma'(u)\sigma'(v)]+\Sigma_{i,j}^{(l+1)}.
    \end{align*} Then, $\Hbb=(\tilde{\Hbb}^{(L)}+\Sigma^{(L)})/2$ is called the neural tangent kernel (NTK) matrix on the context set. Here,  the Gram matrix $\Hbb$ $\in \mathbb{R}^{TK\times TK}$ of the NTK for $L$-layer neural networks is defined recursively from the input layer $l=1$ to the output layer $l=L$. Finally, $\sigma'$ denotes the derivative of the function $\sigma(x)=\max\{x,0\}.$ 
\end{definition}
Furthermore, we define the \textit{effective dimension} of the neural tangent kernel matrix $\Hbb$ as
\begin{align}
    \tilde{d}= \frac{\log {\rm det} (\mathbf{I}+\Hbb/\lambda)}{\log(1+TK/\lambda)} \label{eq:effectivedimension}.
\end{align}
\begin{assumption}\label{assum:3} There exists $\lambda_0>0$ such that
    $\Hbb \geq \lambda_0 \mathbb{I}$. Furthermore, for any $ 1 \leq i \leq TK$, $\Vert \xbf_i \Vert_2 =1$ and $[\xbf_i]_j=[\xbf_i]_{j+d/2}$.
\end{assumption}
The assumption above follows from \cite{zhang2020neural}, with the assumption on context vectors ensuring that $f(\xbf;\thetabf_0)=0$ for any $\xbf$, when initialized as in Delayed NeuralUCB.

We are now ready to present the main result. 
\begin{theorem}\label{thm:maintheorem}
    Let $\tilde{d}$ be the effective dimension, and $\mathbf{h}=[h(\xbf_i)]_{i=1}^{TK}$. Let $\Ebb[\tau]$ denote the mean of the delay distribution that satisfies Assumption~\ref{assum:delays}. 
Under Assumption~\ref{assum:1} and Assumption~\ref{assum:3}, there exists constants $C_1,C_4>0$ such that for any $\delta \in (0,1)$ if
\begin{align*}
    m &\geq {\rm poly}(T,L,K,\lambda^{-1},\lambda_0^{-1},S^{-1},\log(1/\delta)),\\
    \eta &= C_1(mTL+m\lambda)^{-1},
\end{align*}
 $\lambda \geq \max\{1,S^{-2}, \mathcal{L}^2\}$, $\mathcal{L} \geq \Vert g(\xbf;\thetabf)\Vert_2$ for all $\xbf,\thetabf$, and $S \geq \sqrt{2\mathbf{h}^{\top}\mathbf{H}^{-1}\mathbf{h}}$, then with probability at least $1-\delta$, the regret of Algorithm 1 satisfies
\begin{small}
\begin{align}
    R_T &\leq \scalemath{0.9}{\Biggl[\sqrt{T \Bigl( 2 \tilde{d}\log\Bigl( 1+\frac{TK}{\lambda}\Bigr)+2\Bigr)}+ \frac{D_+}{2} \Bigl( 2 \tilde{d}\log\Bigl( 1+\frac{TK}{\lambda}\Bigr)+2\Bigr) \Biggr]} \nonumber \\&   \biggl[2\Bigl( \nu \Bigl[\tilde{d}\log(1+ \frac{TK}{\lambda})+ 2 -2\log \delta\Bigr]^{1/2}  +2\sqrt{\lambda S}\Bigr) \nonumber \\&+2(\lambda+C_4TL)\Bigl(\Bigl(1-\frac{\lambda}{TL}\Bigr)^{J/2}\sqrt{\frac{T}{\lambda}} \Bigr) \biggr]+1, \label{eq:finalregretbound}
\end{align} \end{small} where \begin{align}
    D_{+} =1+2\Ebb[\tau]+D_{\tau}+\psi_{\tau} \label{eq:D+}
\end{align}with
\begin{align*}
  D_{\tau}&= \min \Biggl\{\sqrt{2\alpha^2 \log \frac{3T}{2\delta}}, 2b \log \frac{3T}{2\delta} \Biggr \} \\
    \psi_{\tau}&= \frac{4}{3} \log \frac{3T}{2\delta}+2 \sqrt{2\Ebb[\tau]\log \frac{3T}{2\delta}}.  
\end{align*}
\end{theorem}

 \subsection{Discussion on Theorem~\ref{thm:maintheorem}}
Theorem~\ref{thm:maintheorem} captures the impact of delayed reward observation through the second term $\frac{D_+}{2} \Bigl( 2 \tilde{d}\log\Bigl( 1+TK/\lambda\Bigr)+2\Bigr)$ in the first summation in Eq.~\eqref{eq:finalregretbound}, where the term $D_+$ depends on the expected delay $\Ebb[\tau]$ of the delay distribution. 
 In contrast, the regret bound of the NeuralUCB algorithm in \cite{zhang2020neural} does not include this additional term, as it assumes immediate reward feedback. Consequently, delayed reward feedback leads to increased regret, as expected.

In the regret upper bound in Eq.~\eqref{eq:finalregretbound}, there is a term that depends linearly on $T$: $(\lambda+C_4TL)\Bigl((1-\lambda/(TL))^{J/2}\sqrt{T/\lambda} \Bigr)$. This term characterizes the optimization error of Algorithm~\ref{alg:TrainNN} after $J$  gradient descent iterations. One can set 
$$J = 2 \log \frac{\lambda S}{\sqrt{T}(\lambda+C_4TL)} \frac{TL}{\lambda}$$ as in \cite{zhang2020neural} to ensure that the linear dependence is removed as $(\lambda+C_4TL)\Bigl((1-\lambda/(TL))^{J/2}\sqrt{T/\lambda} \Bigr)\leq \sqrt{\lambda}S$. 
Using the above insight, from Eq.~\eqref{eq:finalregretbound} we deduce that
the additional regret incurred due to delayed reward feedback scales as $\mathcal{O}(D_+ (\tilde{d} \log(1+TK/\lambda))^{3/2})$.

\section{Regret Analysis: Proof of Theorem~\ref{thm:maintheorem}}\label{sec:proof}
The proof of Theorem~\ref{thm:maintheorem} requires several key intermediate lemmas which we discuss here.

The first lemma below shows that for a sufficiently wide neural network model (i.e., when $m$ is sufficiently large), with high probability, the true unknown reward function $h(\xbf_i)$ can be approximated as a linear function of the gradient $g(\xbf_i;\thetabf_0)$ parameterized by $\thetabf^* -\thetabf_0$. 
\begin{lemma}[\cite{zhou2020neural}]
    There exists a positive constant $\bar{C}$ such that for any $\delta \in (0,1)$, if the depth of the neural network satisfies $m \geq \bar{C}T^4K^4L^6\log(T^2K^2L/\delta)/\lambda_0^4$, then with probability at least $1-\delta$, there exists a $\thetabf^* \in \mathbb{R}^p$ such that
   \begin{align*}
       h(\xbf_i)&= g(\xbf_i;\thetabf_0)^{\top} (\thetabf^* - \thetabf_0), \\
       \sqrt{m} \Vert \thetabf^* -\thetabf_0 \Vert_2 &\leq \sqrt{2\mathbf{h}^{\top} \Hbb^{-1} \mathbf{h}},
   \end{align*} for all $i \in [TK]$.
\end{lemma}
The next lemma will characterize how far $\thetabf_t$ is from initialization $\thetabf_0$.
\begin{lemma}\label{lem:key_1}
    There exists positive constants $\bar{C}_1$ and $\bar{C}_2$ such that for any $\delta \in (0,1)$, if the gradient descent step size $\eta \leq \bar{C}_1 (TmL+m\lambda)^{-1}$ and 
    \begin{align*} m &\geq \bar{C}_2 \max\left\{T^7\lambda^{-7} L^{21} (\log m)^3,\right.\\ & \left.\lambda^{-1/2}L^{-3/2} (\log(TKL^2/\delta))^{3/2} \right\}, \end{align*}
    then with probability at least $1-\delta$, we have $\Vert \thetabf_t -\thetabf_0 \Vert_2 \leq 2 \sqrt{|\Ical_t|/(m\lambda)}$  and $$ \Vert \thetabf^* -\thetabf_t\Vert_{\Zbf_t} \leq \frac{\gamma_t}{\sqrt{m}},$$ for all $t \in [T]$, where $\gamma_t$ is defined as in Eq.~\eqref{eq:gamma_t}.
\end{lemma}
\begin{proof}
    The proof relies on the proof of Lemma~5.2 of \cite{zhou2020neural} using the following definitions:
    \begin{align*}
        \bar{\Zbf}_t &= \lambda \mathbf{I}+ \sum_{s=1}^t \mathbb{I}\{s+\tau_s \leq t\} g(x_{s,a_s};\thetabf_0)g(x_{s,a_s};\thetabf_0)^{\top}/m \\
        \bar{\mathbf{b}}_t &= \sum_{s=1}^t \mathbb{I}\{s+\tau_s \leq t\} r_{s,a_s}
 g(x_{s,a_s};\thetabf_0)/\sqrt{m}\\
 \bar{\gamma_t} &= R \sqrt{\log \frac{{\rm det} (\bar{\Zbf}_t)}{{\rm det} (\lambda \mathbf{I})}-2 \log \delta}+ \sqrt{\lambda}{S},\end{align*} where $\mathbb{I}\{q\}$ is an indicator function that evaluates to $1$ if the condition $q$ is true and zero otherwise. The critical change will be in the use of Theorem~2 of \cite{abbasi2011improved} which holds when there are no delayed reward feedbacks. In the presence of delays, this result can replaced with  Lemma~1 of \cite{howson2023delayed}.
\end{proof}

We now quantify the difference in true reward of the optimal action ($h(\xbf_{t,a^*_t})$) and the action taken according to Delayed NeuralUCB ($h(\xbf_{t,a_t})$) at each $t^{\rm th}$ round. The following result, which is a modification of Lemma~5.3 of \cite{zhou2020neural}, quantifies this as a function of the gradient norm, $\Vert g(\xbf_{t,a_t};\thetabf_{t-1})/\sqrt{m}\Vert_{\Zbf_{t-1}^{-1}}$.
\begin{lemma}[\cite{zhou2020neural}]\label{lem:3}
    Let $a^*_t = \arg \max_{a} h(\xbf_{t,a})$. There exists a positive constant $\bar{C}$ such that for any $\delta \in (0,1)$, if $\eta$ and $m$ satisfy the same conditions as in Lemma~\ref{lem:key_1}, then with probability at least $1-\delta$, we have
    \begin{align*}
        h(\xbf_{t,a^*_t})-h(\xbf_{t,a_t}) &\leq 2 \gamma_{t-1} \Vert g(\xbf_{t,a_t};\thetabf_{t-1})/\sqrt{m}\Vert_{\Zbf_{t-1}^{-1}} \\& + \bar{C}(Sm^{-1/6} \sqrt{\log m}t^{1/6} \lambda^{-1/6}L^{7/2}\\&+m^{-1/6}\sqrt{\log m}T^{2/3}\lambda^{-2/3}L^3).
    \end{align*}
\end{lemma}

Using Lemma~\ref{lem:3}, the total regret $R_T$ can be bounded as
\begin{align}
    R_T &=\sum_{t=1}^T h(\xbf_{t,a^*_t}) -h(\xbf_{t,a_t})\nonumber\\&\leq \sum_{t=1}^T  2 \gamma_{t-1} \Vert g(\xbf_{t,a_t};\thetabf_{t-1})/\sqrt{m}\Vert_{\Zbf_{t-1}^{-1}} + \mathcal{D}_T \nonumber\\
    & \leq   2 \gamma_{T}\sum_{t=1}^T \Vert g(\xbf_{t,a_t};\thetabf_{t-1})/\sqrt{m}\Vert_{\Zbf_{t-1}^{-1}} + \mathcal{D}_T \label{eq:inter_1}
\end{align} where $\mathcal{D}_T=\bar{C}(Sm^{-1/6} \sqrt{\log m}T^{7/6} \lambda^{-1/6}L^{7/2}+m^{-1/6}\sqrt{\log m}T^{5/3}\lambda^{-2/3}L^3)$, and the last inequality follows from the monotonicity of $\gamma_t$ with respect to iteration $t$, yielding $\gamma_t \leq \gamma_T$.

Thus, to upper bound $R_T$, we must evaluate the summation $\sum_{t=1}^T \Vert g(\xbf_{t,a_t};\thetabf_{t-1})/\sqrt{m}\Vert_{\Zbf_{t-1}^{-1}}$. In the absence of delayed feedback, this summation can be evaluated using the well-known elliptical potential lemma \cite{abbasi2011improved}. However, elliptical potential lemma requires that the learner update the designer matrix at the end of every iteration with the most recent action. This, however, does not hold true for our designer matrix $\Zbf_t$. To circumvent this issue, we define the following terms:
\begin{align}
       & V_t  = \lambda \mathbf{I}+\sum_{s=1}^t \frac{g(x_{s,a_s};\thetabf_{t-1})g(x_{s,a_s};\thetabf_{t-1})^{\top}}{m}  \nonumber\\
        &W_t = \sum_{s=1}^t \mathbb{I}\{s+\tau_s>t\} \frac{g(x_{s,a_s};\thetabf_{t-1})g(x_{s,a_s};\thetabf_{t-1})^{\top}}{m}, \label{eq:W_t}
    \end{align} such that $V_t =W_t+\Zbf_t$.
The following result uses the above definition to get an upper bound on the term $\sum_{t=1}^T \Vert g(\xbf_{t,a_t};\thetabf_{t-1})/\sqrt{m}\Vert_{\Zbf_{t-1}^{-1}}$.
\begin{lemma}\label{lem:ellpiticalpotential_lemma}
Assume that $\lambda \geq \max\{1,\mathcal{L}^2\}$, where $\mathcal{L} \geq \Vert g(\xbf;\thetabf)/\sqrt{m}\Vert$ for all $\xbf,\thetabf$. There exists positive constant $C_1>0$ such that the following upper bound holds with probability at least $1-2\delta$, 
    \begin{align}
        &\sum_{t=1}^T \Vert \frac{g(\xbf_{t,a_t};\thetabf_{t-1})}{\sqrt{m}} \Vert_{\Zbf_{t-1}^{-1}} \leq \sqrt{T \Bigl( 2 \tilde{d}\log\Bigl( 1+\frac{TK}{\lambda}\Bigr)+\Gamma_1\Bigr)}\nonumber \\&  \quad + \frac{D_+}{2} \Bigl( 2 \tilde{d}\log\Bigl( 1+\frac{TK}{\lambda}\Bigr)+\Gamma_1\Bigr),
    \end{align}  where $D_+$ is as defined in Eq.~\eqref{eq:D+}, and 
    \begin{align}
        \Gamma_1 = 1+C_1m^{-1/6}\sqrt{\log m}L^4T^{5/3}\lambda^{-1/6}. \label{eq:Gamma_1}
    \end{align}
\end{lemma}
\begin{proof}
    Using the definition of $V_t$ and $W_t$ in Eq.~\eqref{eq:W_t},  Lemma~2 of \cite{howson2023delayed} yields that
    \begin{align}
        \Zbf_t^{-1} = V_t^{-1}+ M_t^{-1}, \quad M_t^{-1}=V_t^{-1}W_t\Zbf_t^{-1}. \nonumber
    \end{align}
Using this, we get that
\begin{align}
  &\sum_{t=1}^T \Vert \frac{g(\xbf_{t,a_t};\thetabf_{t-1)}}{\sqrt{m}} \Vert_{\Zbf_{t-1}^{-1}} \nonumber\\&  \leq \sum_{t=1}^T \Vert\frac{g(\xbf_{t,a_t};\thetabf_{t-1})}{\sqrt{m}} \Vert_{V_{t-1}^{-1}} +\sum_{t=1}^T \Vert \frac{g(\xbf_{t,a_t};\thetabf_{t-1})}{\sqrt{m}}\Vert_{M_{t-1}^{-1}} \nonumber\\
  & \leq \sqrt{T \sum_{t=1}^T \Vert \frac{g(\xbf_{t,a_t};\thetabf_{t-1)}}{\sqrt{m}} \Vert^2_{V_{t-1}^{-1}}} + \sum_{t=1}^T \Vert \frac{g(\xbf_{t,a_t};\thetabf_{t-1})}{\sqrt{m}} \Vert_{M_{t-1}^{-1}}, \label{eq:1}
\end{align} where the first inequality follows from triangle inequality, and the second inequality follows from applying Cauchy-Schwarz inequality. 

We now upper bound the second term of Eq.~\eqref{eq:1}. To this end, define $G_t =\sum_{s=1}^t\mathbb{I}\{s+\tau_s > t\}$ as the number of missing rewards at the end of round $t$. Combining Lemma~3, Lemma~4 and Lemma~5 of \cite{howson2023delayed} help to leverage  the  sub-exponential property of the delay distribution (Assumption~\ref{assum:delays}) to yield the following upper bound: with probability at least $1-2\delta$, we have
\begin{align}
    \sum_{t=1}^T \Vert \frac{g(\xbf_{t,a_t};\thetabf_{t-1})}{\sqrt{m}} \Vert_{M_{t-1}^{-1}} \leq \frac{D_{+}}{2} \sum_{t=1}^T \Vert \frac{g(\xbf_{t,a_t};\thetabf_{t-1})}{\sqrt{m}} \Vert_{V_{t-1}^{-1}}^2 \label{eq:2}
\end{align}where $D_+$ is as defined in Eq.~\eqref{eq:D+}.

All thats left now is to upper bound the term $\sum_{t=1}^T \Vert \frac{g(\xbf_{t,a_t};\thetabf_{t-1})}{\sqrt{m}} \Vert_{V_{t-1}^{-1}}^2$. To this end, we can now apply the elliptical potential lemma (Lemma~11 of \cite{abbasi2011improved}) as the matrix $V_t$, defined in \eqref{eq:W_t}, is updated at the end of every iteration with the most recent action. For $\lambda \geq \max(1, \mathcal{L}^2 )$ where $\Vert g(\xbf;\thetabf)\Vert_2 \leq \mathcal{L}$ for all $\xbf,\thetabf$, we then have that 
\begin{align}
   & \sum_{t=1}^T \Vert \frac{g(\xbf_{t,a_t};\thetabf_{t-1})}{\sqrt{m}} \Vert_{V_{t-1}^{-1}}^2 \leq 2 \log \frac{{\rm det}(V_T)}{{\rm det}(\lambda \mathbf{I})}\nonumber\\
    & \leq 2 \log \frac{{\rm det}(\bar{V}_T)}{{\rm det}(\lambda \mathbf{I})}+C_1m^{-1/6}\sqrt{\log m}L^4T^{5/3}\lambda^{-1/6}\nonumber \\
    &\leq 2 \tilde{d}\log\Bigl( 1+TK/\lambda\Bigr)+1+C_1m^{-1/6}\sqrt{\log m}L^4T^{5/3}\lambda^{-1/6}, \label{eq:logratiobound}
\end{align} where the second and last inequalities follow from equation (B.18) of \cite{zhou2020neural} with $\tilde{d}$ denoting the effective dimension, and  $$\bar{V}_t = \frac{1}{m} \lambda \mathbf{I}+\sum_{s=1}^t g(x_{s,a_s};\thetabf_0)g(x_{s,a_s};\thetabf_0)^{\top}.$$
Substituting \eqref{eq:logratiobound} and \eqref{eq:2} in \eqref{eq:1} yields the required bound.
\end{proof}
Now we present the proof of our main result below.

\noindent{\em Proof of Theorem~\ref{thm:maintheorem}:}
Using Lemma~\ref{lem:ellpiticalpotential_lemma} in Eq.~\eqref{eq:inter_1}, we get that the cumulative regret
\begin{align}
    R_T 
    & \leq 2\gamma_T \Biggl[\sqrt{T \Bigl( 2 \tilde{d}\log\Bigl( 1+\frac{TK}{\lambda}\Bigr)+\Gamma_1\Bigr)}\nonumber \\&  \quad + \frac{D_+}{2} \Bigl( 2 \tilde{d}\log\Bigl( 1+TK/\lambda\Bigr)+\Gamma_1\Bigr) \Biggr]+ \mathcal{D}_T. \label{eq:upperbound_1}
\end{align}

We now upper bound $\gamma_T$ as defined in Eq.~\eqref{eq:gamma_t}. Recalling $\Ical_t =\{s: s+\tau_s \leq t\}$, we make the following observations: $|\Ical_T| \leq T$ and the designer matrix $\Zbf_T$ satisfies that $\Zbf_T \prec V_T$, where $V_t$ is as defined in Eq.~\eqref{eq:W_t}. The latter yields that $$\log  \frac{{\rm det}(\Zbf_T)}{{\rm det}(\lambda \mathbf{I})}\leq \log  \frac{{\rm det}(V_T)}{{\rm det}(\lambda \mathbf{I})} \leq \tilde{d}\log(1+TK/\lambda)+ \Gamma_1, $$ with the last upper bound following from Eq.~\eqref{eq:logratiobound}. Using all these together yields the following upper bound on $\gamma_T$,
\begin{align*}
    \gamma_T 
    & \leq \Gamma_2\Bigl( \nu \Bigl[\tilde{d}\log(1+TK/\lambda)+ \Gamma_1-2\log \delta\Bigr]^{1/2}  +\sqrt{\lambda S}\Bigr)\nonumber\\&+(\lambda+C_3TL)\Bigl((1-\eta m\lambda)^{J/2}\sqrt{T/\lambda} \\&+\Gamma_3(1+\sqrt{T/\lambda}) \Bigr),
\end{align*}where we have defined $\Gamma_1$ as in Eq.~\eqref{eq:Gamma_1} and 
\begin{align*}
    \Gamma_2 & = \sqrt{1+C_3m^{-1/6}\sqrt{\log m}L^4T^{7/6}\lambda^{-7/6)}}\\
    \Gamma_3& =m^{-1/6}\sqrt{\log m}L^{7/2}T^{5/3}\lambda^{-5/3}.
\end{align*}
Using the derived bound on $\gamma_T$ in Eq.~\eqref{eq:upperbound_1}, we get
\begin{align*}
    &R_T \leq \biggl[2 \Gamma_2\Bigl( \nu \Bigl[\tilde{d}\log(1+TK/\lambda)+ \Gamma_1 -2\log \delta\Bigr]^{1/2}  +\sqrt{\lambda S}\Bigr) \\
    &+2(\lambda+C_3TL)\Bigl((1-\eta m\lambda)^{J/2}\sqrt{T/\lambda} +\Gamma_3(1+\sqrt{T/\lambda}) \Bigr)\biggr] \\
    & \Biggl[\sqrt{T \Bigl( 2 \tilde{d}\log\Bigl( 1+\frac{TK}{\lambda}\Bigr)+\Gamma_1\Bigr)}\nonumber \\
    &  \quad + \frac{D_+}{2} \Bigl( 2 \tilde{d}\log\Bigl( 1+TK/\lambda\Bigr)+\Gamma_1\Bigr) \Biggr]+ \mathcal{D}_T\\
    & \leq \biggl[2\Bigl( \nu \Bigl[\tilde{d}\log(1+TK/\lambda)+ 2 -2\log \delta\Bigr]^{1/2}  +2\sqrt{\lambda S}\Bigr)\\
    &+2(\lambda+C_4TL)\Bigl((1-\eta m\lambda)^{J/2}\sqrt{T/\lambda} \Bigr) \biggr]\\
    &\Biggl[\sqrt{T \Bigl( 2 \tilde{d}\log\Bigl( 1+\frac{TK}{\lambda}\Bigr)+2\Bigr)}\nonumber \\
    &  \quad + \frac{D_+}{2} \Bigl( 2 \tilde{d}\log\Bigl( 1+TK/\lambda\Bigr)+2\Bigr) \Biggr]+1,
\end{align*}
where the last inequality holds for sufficiently large $m$.  This completes the proof of  Theorem~\ref{thm:maintheorem}.
\qed
\section{Experiments and Discussion}\label{sec:experiments}
\begin{figure*}[h!]
\centering
\subfloat{\includegraphics[width=1\columnwidth,scale=0.4, clip=true, trim=0in 0in 0.6in 0.2in]{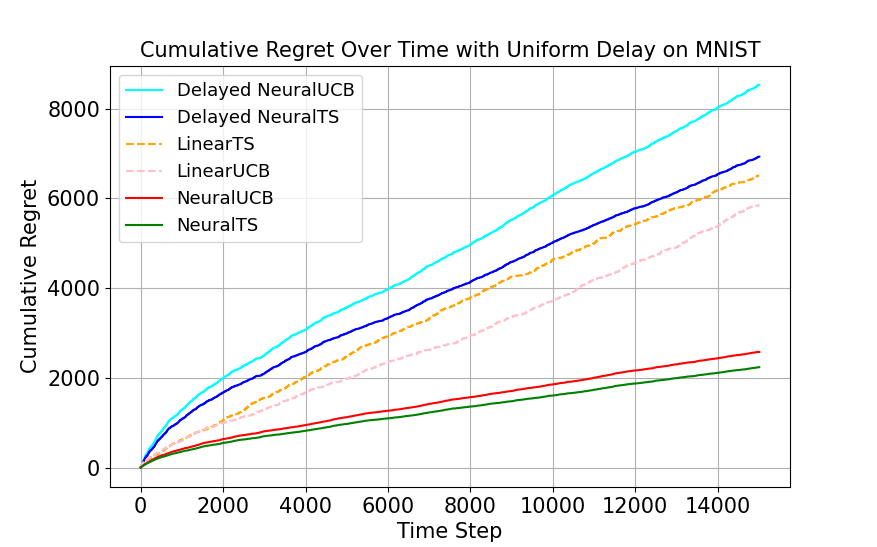}}
\hfil
\subfloat{\includegraphics[width=1\columnwidth,scale=0.4,clip=true, trim=0in 0in 0.6in 0.2in]{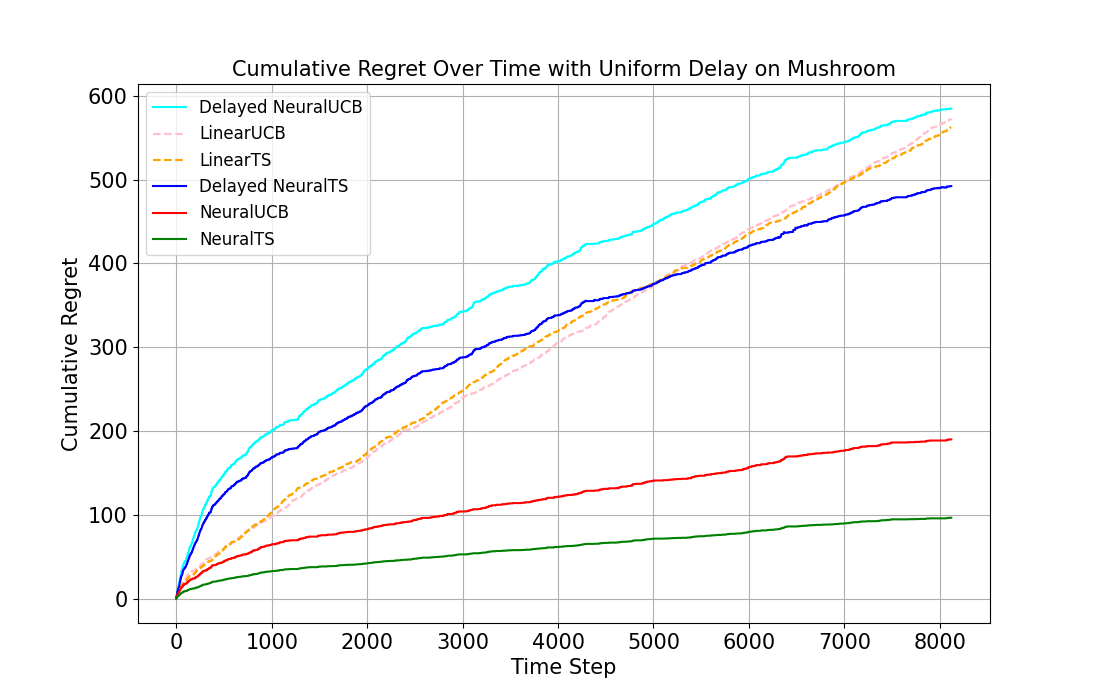}}
\caption{{\small Comparison of the cumulative regret of the algorithms with no delayed feedback -- LinTS/UCB, Neural-TS/UCB -- with our proposed algorithms Delayed Neural-UCB/TS under uniform delay with $\Ebb[\tau]=30$, as a function of the number of iterations on (left) MNIST and (right) Mushroom datasets. We run 5 experiments and plot the mean regret.}}\label{fig:uniformdelay_mnistmushroom}

\end{figure*}
\begin{figure*}[h!]
\centering
\subfloat{\includegraphics[width=1\columnwidth,scale=0.4,clip=true, trim=0in 0.2in 0.4in 0.6in]{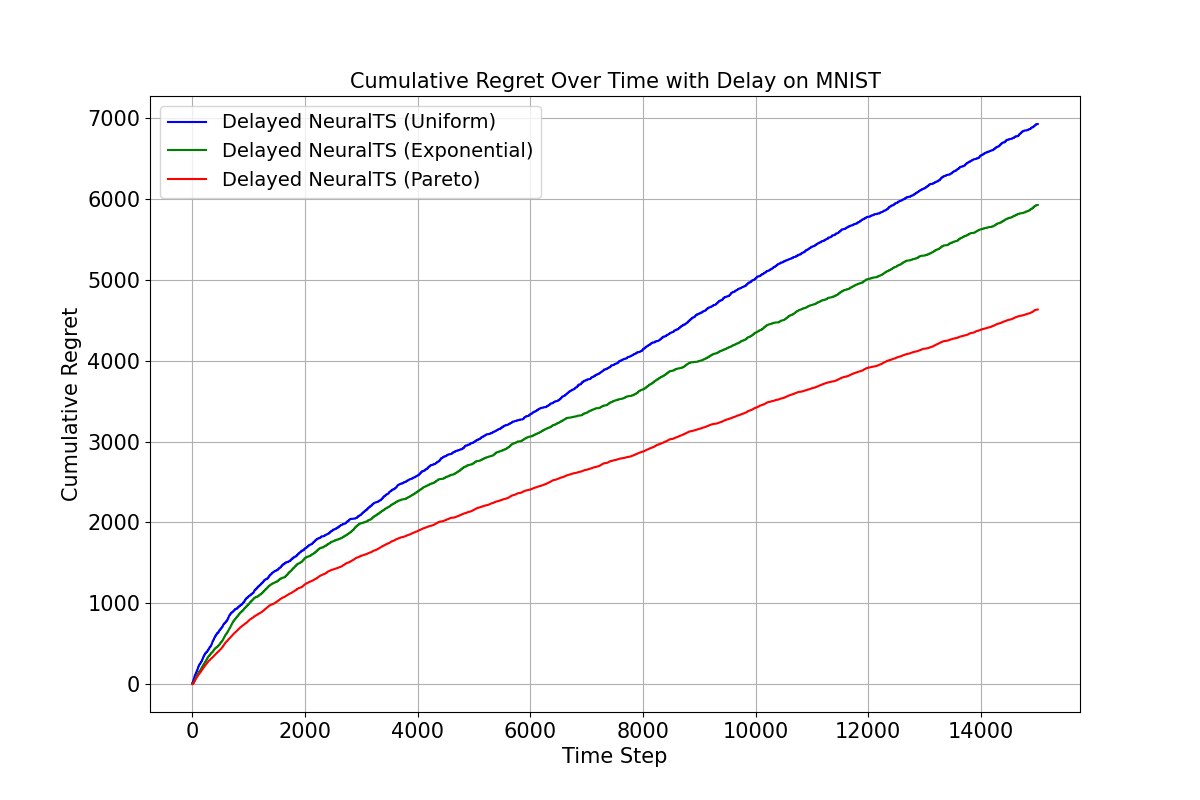}}
\hfil
\subfloat{\includegraphics[width=1\columnwidth,scale=0.4,clip=true, trim=0in 0.2in 0.4in 0.6in]{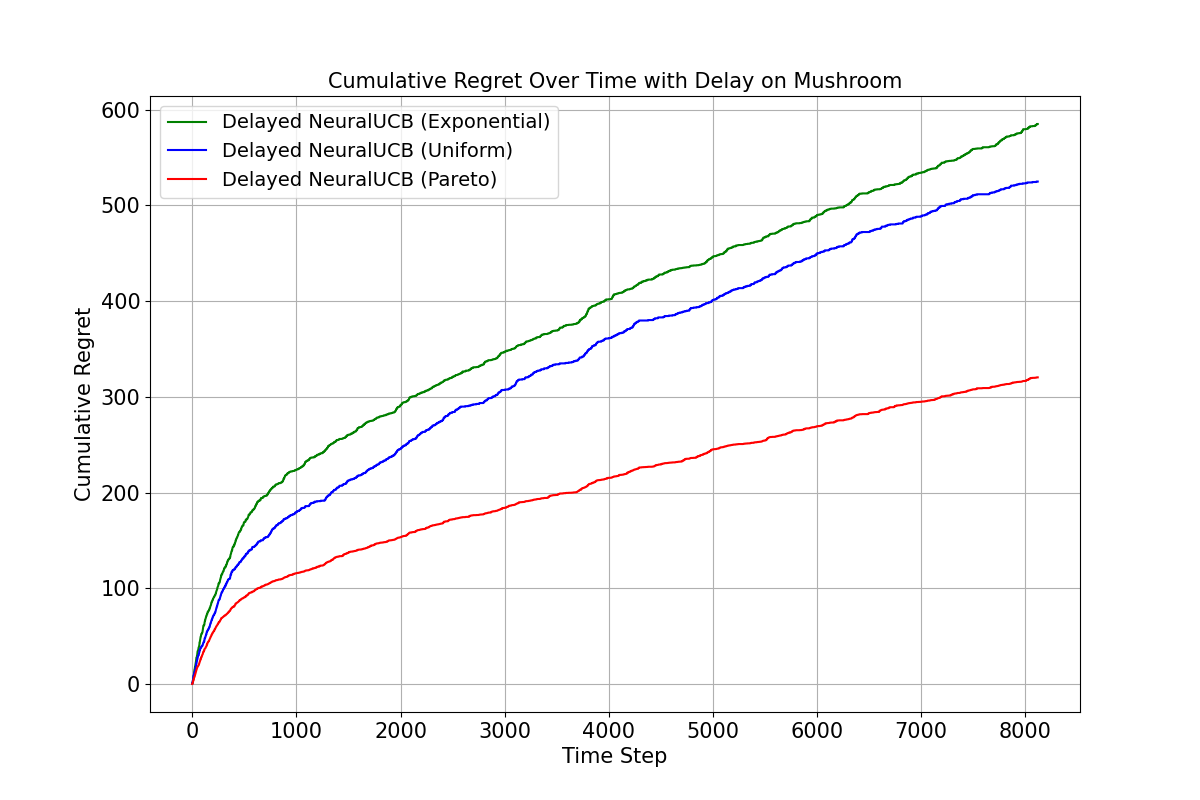}}
\caption{{\small (Left) Comparison of the cumulative regret of the Delayed NeuralTS under uniform, exponential and Pareto delays with $\Ebb[\tau]=30$ as a function of the number of iterations on MNIST. (Right) Comparison of the cumulative regret of the Delayed NeuralUCB under the three delays on Mushroom.}}\label{fig:differentdelays}
\vspace{-0.6cm}
\end{figure*}
In this section, we evaluate the performance of the proposed Delayed NeuralUCB algorithm under three different delay distributions \cite{howson2023delayed}: 
Exponential$(\Delta)$ with $\Delta =1/\Ebb[\tau]$, Uniform$(0,B)$ with $B=2\Ebb[\tau]$ and Pareto$(a,x_m=1)$ with $a=(1+\Ebb[\tau])/\Ebb[\tau]$, for some fixed $\Ebb[\tau]$. Note that Assumption~\ref{assum:delays} holds for uniform and exponential distributions, but not for Pareto distribution.

\subsection{Dataset}
We evaluate our algorithms on real-world datasets -- MNIST \cite{lecun1998gradient} and Mushroom \cite{mushroom_73}. Note that these are $K$-class classification datasets. To convert the classification problem to $K$-armed contextual bandit problem, we adopt the approach of \cite{zhou2020neural}. Specifically, each data sample $\xbf \in \mathbb{R}^d$ is transformed into $K$ contextual vectors: $\xbf^{(1)}=(\xbf,\obf,\cdots,\obf)$, $\hdots$, $\xbf^{(K)}=(\obf,\obf,\hdots,\xbf) \in \Rbb^{dK}$ corresponding to each of the $K$ classes. The agent receives a reward of $1$ if he/she classifies the context correctly and $1$ otherwise. We further add a small mean-zero Gaussian noise with variance $0.001$ to the above reward. Note that for MNIST data set ($28\times 28\times 1$ pixels) consisting of $K=10$ classes, the above transformation yields contextual vectors of dimension $28\times 28\times 10=7840$. In contrast, the mushroom dataset consists of $d=22$ dimensional input samples and $K=2$ classes, results in contextual vectors of dimension $44$. The high-dimensionality of contextual vectors studied here reveals interesting insights that are lacking in the existing works \cite{howson2023delayed} on delayed reward feedback.
\subsection{Benchmark Algorithms}
We now describe various algorithms used to benchmark the performance of our Delayed NeuralUCB algorithm against. To this end, we first consider 
\textit{linear contextual UCB \cite{chu2011contextual} and linear contextual Thompson Sampling (TS) \cite{agrawal2013thompson}} algorithms. These algorithms assume linear reward model for CBs and use UCB and TS respectively for exploration. For experiments on Mushroom and MNIST dataset, we analyze how well these standard algorithms perform when no delay in reward feedback is assumed, thereby setting a benchmark for best-case cumulative regret bound. 

However, the assumption of linear reward  models may prove insufficient to capture complex correlations between the reward-action-context tuples. Consequently, we then consider
\textit{NeuralUCB \cite{zhou2020neural} and NeuralTS \cite{zhang2020neural}} algorithms that leverage highly non-linear neural network models to better capture complex correlations, while employing UCB or TS-based explorations. No delay in reward feedback is assumed, and these algorithms set another benchmark for best-case cumulative regret. 

The last algorithm we consider is the \textit{Delayed NeuralTS} algorithm,  which is an adaptation of NeuralTS to account for delayed reward feedback, similar to the proposed Delayed NeuralUCB algorithm. At each $t^{\rm th}$ iteration, the algorithm uses TS-based exploration to update a posterior distribution of the reward model using the history of observed data $\{(\xbf_{s,a_s},a_s,r_s)\}_{s \in \mathcal{I}_t}$ with complete reward information. Specifically, for each action $a \in [1,\hdots,K]$, the following steps replace lines 5--8 of Algorithm~\ref{alg:neuralUCB}: TS algorithm samples an estimated reward $\tilde{r}_{t,a} \sim \mathcal{N}(f(\xbf_{t,a};\thetabf_{t-1}), \nu^2 \sigma^2_{t,a})$ from the Gaussian posterior whose variance is determined by the exploration parameter $\nu$ and  \begin{align} \sigma^2_{t,a} = \lambda g(\xbf_{t,a}; \theta_{t-1})^\top \Zbf_{t-1}^{-1} g(\xbf_{t,a}; \theta_{t-1}) / m, \label{eq:TS_variance} \end{align} with $\Zbf_t$ and $\thetabf_t$ updated as in Algorithm~\ref{alg:neuralUCB}. Subsequently, the action $a_t =\arg \max_a \tilde{r}_{t,a}$ that maximizes the estimated reward is selected.

\subsection{Experiment and Results}

For all our algorithms, we choose a two-layer neural network $f(\xbf;\theta)=\sqrt{m}\Wbf_2 \sigma(\Wbf_1 \xbf)$ with network width $m= 128 $. We used $\gamma_t$ as in Eq.~\eqref{eq:gamma_t} with $C_1=C_2=C_3=1$.  Other parameters were fixed as $\nu= 1$, $\lambda= 1$, $\delta= 0.05$ and $S= 0.0001$. For training the neural network, we used stochastic gradient descent with a batch size of 64, $J=t$ at round $t$, and a fixed learning rate of $\eta=0.001$.

Figure~\ref{fig:uniformdelay_mnistmushroom} compares the performance of our proposed delayed neural contextual algorithms -- Delayed NeuralUCB/TS  under uniform delay with $\Ebb[\tau]=30$ -- with standard algorithms Neural TS/UCB and Lin-UCB/TS under no delays. The standard algorithms are unaffected by delays, yielding the optimal performance to benchmark our algorithms against. Furthermore, as seen from Figure~\ref{fig:uniformdelay_mnistmushroom}, the linear TS/UCB algorithms, that fit linear reward models, fail to model complex reward relationships in MNIST/Mushroom classification problems compared to neural TS/UCB, resulting in higher regret.

Experiments on MNIST Figure~\ref{fig:uniformdelay_mnistmushroom} (left), with 7840-dimensional context vectors, show that Delayed NeuralTS algorithm, in particular, achieves regret close to LinearTS algorithm. In fact, this performance can be seen to improve on  Mushroom dataset (right), consisting of 44-dimensional context vectors, with Delayed NeuralUCB achieving comparable performance to LinearUCB/TS algorithms. 
Furthermore, across all algorithms, TS-based exploration seems to empirically yield lower regret than UCB-based exploration. 

Figure~\ref{fig:differentdelays} compares the regret of our proposed algorithms under three delay distributions -- uniform, exponential, and Pareto, as a function of the iterations. Figure~\ref{fig:differentdelays} (left) compares Delayed NeuralTS  on MNIST, and Figure~\ref{fig:differentdelays} (right) compares Delayed NeuralUCB  on Mushroom. The experiments show that our algorithms can effectively handle different delays. Although Pareto distribution does not satisfy Assumption~\ref{assum:delays}, our algorithm still performs well under Pareto delays, incurring lower cumulative regret than uniform and exponential.
\vspace{-0.1 cm}
\section{Conclusion}\label{sec:conclusion}
Neural contextual bandits can handle complex, high-dimensional data efficiently, optimizing decision-making through adaptive learning from contextual information, making them applicable in many real-world problems.
This work introduced a new algorithm, called Delayed NeuralUCB, for neural contextual bandits that are tailored to handle delayed reward feedback. We show that the cumulative regret of the algorithm scales as $O(\tilde{d} \sqrt{T \log T}+  \tilde{d}^{3/2} D_+\log(T)^{3/2})$. Additionally, we also propose a variant of the algorithm, called Delayed NeuralTS, that uses Thompson Sampling based exploration. The proposed algorithms have been empirically demonstrated to effectively handle different types of delays on diverse real-world datasets.
Future works include deriving theoretical regret guarantees for Delayed NeuralTS algorithm, and handling different delayed reward feedback mechanisms, including delayed aggregated reward feedback.
\bibliography{ref}

\begin{thebibliography}{10}

\bibitem{durand2018contextual}
A.~Durand, C.~Achilleos, D.~Iacovides, K.~Strati, G.~D. Mitsis, and J.~Pineau,
  ``Contextual bandits for adapting treatment in a mouse model of de novo
  carcinogenesis,'' in {\em Machine learning for healthcare conference},
  pp.~67--82, PMLR, 2018.

\bibitem{shen2015portfolio}
W.~Shen, J.~Wang, Y.-G. Jiang, and H.~Zha, ``Portfolio choices with orthogonal
  bandit learning,'' in {\em Twenty-fourth international joint conference on
  artificial intelligence}, 2015.

\bibitem{mcinerney2018explore}
J.~McInerney, B.~Lacker, S.~Hansen, K.~Higley, H.~Bouchard, A.~Gruson, and
  R.~Mehrotra, ``Explore, exploit, and explain: personalizing explainable
  recommendations with bandits,'' in {\em Proceedings of the 12th ACM
  conference on recommender systems}, pp.~31--39, 2018.

\bibitem{lin2024distributed}
J.~Lin, K.~A. Sajeevan, B.~Acharya, S.~Moothedath, and R.~Chowdhury,
  ``Distributed stochastic contextual bandits for protein drug interaction,''
  in {\em ICASSP 2024-2024 IEEE International Conference on Acoustics, Speech
  and Signal Processing (ICASSP)}, pp.~7160--7164, IEEE, 2024.

\bibitem{lin2022stochastic}
J.~Lin, X.~Y. Lee, T.~Jubery, S.~Moothedath, S.~Sarkar, and
  B.~Ganapathysubramanian, ``Stochastic conservative contextual linear
  bandits,'' {\em IEEE Conference on Decision and Control}, 2022.

\bibitem{chu2011contextual}
W.~Chu, L.~Li, L.~Reyzin, and R.~Schapire, ``Contextual bandits with linear
  payoff functions,'' in {\em Proceedings of the fourteenth international
  conference on artificial intelligence and statistics}, pp.~208--214, JMLR
  Workshop and Conference Proceedings, 2011.

\bibitem{agrawal2013thompson}
S.~Agrawal and N.~Goyal, ``Thompson sampling for contextual bandits with linear
  payoffs,'' in {\em International conference on machine learning},
  pp.~127--135, PMLR, 2013.

\bibitem{filippi2010parametric}
S.~Filippi, O.~Cappe, A.~Garivier, and C.~Szepesv{\'a}ri, ``Parametric bandits:
  The generalized linear case,'' {\em Advances in neural information processing
  systems}, vol.~23, 2010.

\bibitem{zhou2020neural}
D.~Zhou, L.~Li, and Q.~Gu, ``Neural contextual bandits with ucb-based
  exploration,'' in {\em International Conference on Machine Learning},
  pp.~11492--11502, PMLR, 2020.

\bibitem{zhang2020neural}
W.~Zhang, D.~Zhou, L.~Li, and Q.~Gu, ``Neural thompson sampling,'' {\em arXiv
  preprint arXiv:2010.00827}, 2020.

\bibitem{joulani2013online}
P.~Joulani, A.~Gyorgy, and C.~Szepesv{\'a}ri, ``Online learning under delayed
  feedback,'' in {\em International conference on machine learning},
  pp.~1453--1461, PMLR, 2013.

\bibitem{pike2018bandits}
C.~Pike-Burke, S.~Agrawal, C.~Szepesvari, and S.~Grunewalder, ``Bandits with
  delayed, aggregated anonymous feedback,'' in {\em International Conference on
  Machine Learning}, pp.~4105--4113, PMLR, 2018.

\bibitem{vernade2017stochastic}
C.~Vernade, O.~Capp{\'e}, and V.~Perchet, ``Stochastic bandit models for
  delayed conversions,'' {\em arXiv preprint arXiv:1706.09186}, 2017.

\bibitem{mandel2015queue}
T.~Mandel, Y.-E. Liu, E.~Brunskill, and Z.~Popovi{\'c}, ``The queue method:
  Handling delay, heuristics, prior data, and evaluation in bandits,'' in {\em
  Proceedings of the AAAI Conference on Artificial Intelligence}, vol.~29,
  2015.

\bibitem{zhou2019learning}
Z.~Zhou, R.~Xu, and J.~Blanchet, ``Learning in generalized linear contextual
  bandits with stochastic delays,'' {\em Advances in Neural Information
  Processing Systems}, vol.~32, 2019.

\bibitem{blanchet2024delay}
J.~Blanchet, R.~Xu, and Z.~Zhou, ``Delay-adaptive learning in generalized
  linear contextual bandits,'' {\em Mathematics of Operations Research},
  vol.~49, no.~1, pp.~326--345, 2024.

\bibitem{howson2023delayed}
B.~Howson, C.~Pike-Burke, and S.~Filippi, ``Delayed feedback in generalised
  linear bandits revisited,'' in {\em International Conference on Artificial
  Intelligence and Statistics}, pp.~6095--6119, PMLR, 2023.

\bibitem{kassraie2022neural}
P.~Kassraie and A.~Krause, ``Neural contextual bandits without regret,'' in
  {\em International Conference on Artificial Intelligence and Statistics},
  pp.~240--278, PMLR, 2022.

\bibitem{xu2020neural}
P.~Xu, Z.~Wen, H.~Zhao, and Q.~Gu, ``Neural contextual bandits with deep
  representation and shallow exploration,'' {\em arXiv preprint
  arXiv:2012.01780}, 2020.

\bibitem{zhu2023scalable}
Z.~Zhu and B.~Van~Roy, ``Scalable neural contextual bandit for recommender
  systems,'' in {\em Proceedings of the 32nd ACM International Conference on
  Information and Knowledge Management}, pp.~3636--3646, 2023.

\bibitem{salgia2023provably}
S.~Salgia, ``Provably and practically efficient neural contextual bandits,'' in
  {\em International Conference on Machine Learning}, pp.~29800--29844, PMLR,
  2023.

\bibitem{jacot2018neural}
A.~Jacot, F.~Gabriel, and C.~Hongler, ``Neural tangent kernel: Convergence and
  generalization in neural networks,'' {\em Advances in neural information
  processing systems}, vol.~31, 2018.

\bibitem{abbasi2011improved}
Y.~Abbasi-Yadkori, D.~P{\'a}l, and C.~Szepesv{\'a}ri, ``Improved algorithms for
  linear stochastic bandits,'' {\em Advances in neural information processing
  systems}, vol.~24, 2011.

\bibitem{lecun1998gradient}
Y.~LeCun, L.~Bottou, Y.~Bengio, and P.~Haffner, ``Gradient-based learning
  applied to document recognition,'' {\em Proceedings of the IEEE}, vol.~86,
  no.~11, pp.~2278--2324, 1998.

\bibitem{mushroom_73}
``{Mushroom}.'' UCI Machine Learning Repository, 1981.
\newblock {DOI}: https://doi.org/10.24432/C5959T.

\end{thebibliography}
\bibliographystyle{ieeetr}
\newpage
\end{document}